\documentclass{article}

\usepackage[preprint]{neurips_2023}

\usepackage[utf8]{inputenc} 
\usepackage[T1]{fontenc}    
\usepackage[colorlinks,
linkcolor=red,
anchorcolor=blue,
citecolor=blue]{hyperref}
\usepackage{url}            
\usepackage{booktabs}       
\usepackage{amsfonts}       
\usepackage{nicefrac}       
\usepackage{microtype}      
\usepackage{xcolor}         
\usepackage{csquotes}
\usepackage{array}
\usepackage{multirow}
\usepackage{pifont}
\usepackage{paralist,bbding}
\usepackage{enumitem}
\usepackage{amsmath}
\usepackage{tikz}
\usepackage{needspace}
\usepackage{lipsum}
\usepackage[table]{colortbl}
\usepackage{threeparttable}
\usepackage[square, comma, sort&compress, numbers]{natbib}

\usepackage{amsmath,amsfonts,amssymb}
\usepackage{amsthm}
\usepackage{graphicx}
\usepackage{color}
\usepackage[normalem]{ulem}
\usepackage{diagbox}
\usepackage{makecell}
\usepackage{multicol}
\usepackage{multirow}
\usepackage{tabularx}
\usepackage{paralist}
\usepackage{array}
\usepackage{setspace}
\usepackage{wrapfig}

\usepackage{listings}
\usepackage{color}
\usepackage[misc]{ifsym}

\definecolor{dkgreen}{rgb}{0,0.6,0}
\definecolor{gray}{rgb}{0.5,0.5,0.5}
\definecolor{mauve}{rgb}{0.58,0,0.82}

\newcommand\ie{\emph{i.e.}}

\title{You Only Scan Once: Efficient Multi-dimension Sequential Modeling with LightNet}

\author{
{
$^1$Zhen Qin, $^3$Yuxin Mao, $^2$Xuyang Shen, $^2$Dong Li, $^4$Jing Zhang, $^3$Yuchao Dai,
$^2$Yiran Zhong$^\textrm{\Letter}$
}\\
$^1$Taptap, $^2$OpenNLPLab, Shanghai AI Laboratory,\\$^3$Northwestern Polytechnical University, $^4$Australian National University \\
\texttt{https://github.com/OpenNLPLab/LightNet} 
}

\newcommand\blfootnote[1]{%
  \begingroup
  \renewcommand\thefootnote{}\footnote{#1}%
  \addtocounter{footnote}{-1}%
  \endgroup
}

\begin{document}

\maketitle

\begin{abstract}
Linear attention mechanisms have gained prominence in causal language models due to their linear computational complexity and enhanced speed. However, the inherent decay mechanism in linear attention presents challenges when applied to multi-dimensional sequence modeling tasks, such as image processing and multi-modal learning. In these scenarios, the utilization of sequential scanning to establish a global receptive field necessitates multiple scans for multi-dimensional data, thereby leading to inefficiencies. This paper identifies the inefficiency caused by a \enquote{multiplicative} linear recurrence and proposes an efficient alternative \enquote{additive} linear recurrence to avoid the issue, as it can handle multi-dimensional data within a single scan. We further develop an efficient multi-dimensional sequential modeling framework called LightNet based on the new recurrence. Moreover, we present two new multi-dimensional linear relative positional encoding methods, MD-TPE and MD-LRPE to enhance the model's ability to discern positional information in multi-dimensional scenarios. Our empirical evaluations across various tasks, including image classification, image generation, bidirectional language modeling, and autoregressive language modeling, demonstrate the efficacy of LightNet, showcasing its potential as a versatile and efficient solution for multi-dimensional sequential modeling. \blfootnote{\noindent $^\textrm{\Letter}$ Indicates corresponding author (Email address: \textit{zhongyiran@gmail.com}).}
\end{abstract}

\section{Introduction}
Linear attention has emerged as an effective alternative to softmax attention due to its linear computational complexity and enhanced processing speed, especially in causal language models~\citep{peng_rwkv4_2024,2307.14995}. The benefits of linear attention largely depend on its decay mechanism~\citep{peng_rwkv4_2024,2307.14995,retnet}, which prevents attention dilution~\citep{qin_transnormer_emnlp_2022} and facilitates global interaction among tokens. However, the decay mechanism presents two primary issues: 
First, the decay mechanism is not easily applicable to high-dimensional inputs due to the need for multiple sequential scans to establish a global multi-dimensional receptive field, which reduces computational efficiency~\citep{duan_visionRWKV_2024,zhu_visionmamba_2024}. 
Additionally, without the decay mechanism, linear attention lacks positional awareness during computations, leading to decreased performance~\citep{qin_transnormer_emnlp_2022}. In light of these challenges, we are investigating the feasibility of reducing sequential scans for multi-dimensional scenarios while preserving performance.

We first analyze the types of linear recurrence and divide them into two categories: \emph{multiplicative} and \emph{additive}. In multiplicative recurrence, the decay rate is dependent only on the current moment, making it impossible to obtain information about subsequent moments with a single scan. By taking image processing as an example, using multiplicative recurrence will require at least two scans to retrieve the global information~\citep{duan_visionRWKV_2024,zhu_visionmamba_2024}.  Conversely, in additive recurrence, the decay rate depends on all moments through the summation of the importance score of each moment, enabling it to gather global information in a single scan.

It is important to note that in non-causal situations, additive recurrence is permutation-invariant, which means it lacks local precedence and therefore diminishes the capture of positional information. To overcome this limitation, we put forth a new approach to positional encoding called Multi-Dimensional Toeplitz Positional Encoding (MD-TPE). This method utilizes the mathematical properties of the Toeplitz matrix to embed relative positional information with linear time complexity, thus ensuring efficiency in multi-dimensional scenarios. Additionally, we expand the Linearized Relative Positional Encoding (LRPE)~\citep{qin2023linearized} to high-dimensional scenarios, resulting in the creation of Multi-Dimensional Linearized Relative Positional Encoding (MD-LRPE).

We then present LightNet, a new multi-dimensional linear attention model built on additive recurrence. LightNet features a pioneering decay mechanism, allowing for efficient single-scan processing of high-dimensional sequential data. Furthermore, it integrates highly effective multi-dimensional position encoding such as MD-TPE and MD-LRPE to precisely capture positional information.

We conduct several evaluations of the performance of our proposed LightNet on a range of tasks, including image generation, image classification, image generation, bidirectional language modeling, and autoregressive language modeling. LightNet performs comparably or better than its competitors across all tasks.

We summarize our main contributions as follows:
\begin{compactitem}
    \item We analyze the types of linear recurrence, dividing them into two types: \emph{multiplicative} and \emph{additive}, where the additive type can obtain global information in a single scan.
    \item We propose two multi-dimensional position encoding strategies, MD-TPE and MD-LRPE, to effectively capture positional information in multi-dimensional scenarios.
    \item We propose LightNet, a new multi-dimensional linear attention model that can process high-dimensional sequences in a single scan.
    \item We conduct thorough evaluations to assess the efficiency and efficacy of LightNet for multi-dimensional sequential modeling tasks. The LightNet demonstrates competitive performance in all scenarios.
\end{compactitem}

\section{Related Work}
\vspace{-3mm}
\paragraph{Linear Attention.}
The linear attention mechanism has greatly advanced deep learning, particularly in natural language processing, by providing a scalable solution for long input sequences and reducing the computational demands of traditional attention models~\cite{choromanski_performer_iclr_2020,katharopoulos_icml_transformerrnn_icml_2020,qin_cosformer_iclr_2021,qin_tnn_2022}. However, despite its faster training speeds, linear attention's performance still falls short of softmax attention due to the attention dilution issue~\citep{qin_transnormer_emnlp_2022}. The TNL/RetNet~\cite{qin_transnormer_emnlp_2022,2307.14995} introduces a decay mechanism to address this problem. Additionally, GLA~\cite{yang_GLA_2023} incorporating gating mechanisms show the potential to enhance linear attention models. 

\vspace{-3mm}
\paragraph{State Space Model.}
State Space Models (SSMs) are increasingly crucial in sequence modeling due to their structured approach to capturing temporal dynamics through latent variables. 
The S4 model~\cite{gu_S4_ICLR_2021} enhances state space modeling for long sequences by leveraging structured spaces to improve computational efficiency and tackle complex dynamics. 
With additional parameterizing and initializing diagonal state space strategy~\cite{gu_S4D_nips_2022}, the SSMs can achieve comparable performance to naive transformers.
Furthermore, the Gated State Space (GSS) model~\cite{mehta_gss_iclr_2023} introduces a gating mechanism to SSMs, which is particularly effective for long-range language modeling by allowing nuanced control over information flow. 
The S5 model~\cite{smith_s5_iclr_2022} reduces complexity using \enquote{scan} while maintaining the capability to handle intricate sequences.
However, directly extending the SSM to multi-dimensional input usually requires multiple sequential scans, which will reduce the computational efficiency~\citep{zhu_visionmamba_2024}.

\vspace{-3mm}
\paragraph{Linear RNN.}
Linear RNNs employ element-wise recursion for sequence modeling, and due to their linear recursive form, they can be accelerated using parallel scans~\citep{iclr18}. 
At their core is the decay mechanism, where RWKV-4/LRU~\cite{peng_rwkv4_2024,lru} utilizes data-independent decay.
HGRN~\cite{qin_hgrn_nips_2024,qin2024hgrn2} leverage data-dependent decay to enhance performance. 
Linear RNNs have shown considerable potential in language modeling and long-sequence modeling tasks.
\vspace{-3mm}
\paragraph{Multi-dimensional Tasks with Linear Complexity Model.}
The development of linear attention in language models has led to its extension into multi-dimensional tasks. Building upon the cosFormer framework~\cite{qin_cosformer_iclr_2021}, VVT~\cite{sun_vvt_pami_2023} explores a local prior of 2D linear attention and applies it to image classification tasks. 
Vim~\cite{zhu_visionmamba_2024} and Vision-RWKV~\cite{duan_visionRWKV_2024} utilize a sequential scan mechanism to expand Mamba~\cite{gu2023mamba} and RWKV~\cite{peng_rwkv_emnlp_2023} for image classification. 
Additionally, leveraging the benefits structure of the diffusion transformer~\cite{peebles_dit_iccv_2023} in image generation, several works have extended linear complexity models into 2D space~\cite{fei_dis_2024,fei_diffusionRWKV_2024,Yan_DiffusionSSM_2023,hu2024zigma} to replace the traditional transformer architecture, achieving efficient image generation.
However, some of these tasks encounter issues with inadequate performance. 
Moreover, frequent sequential scans can compromise the efficiency of the model.

\section{Linear Recurrence in Multi-dimensional Space}
\vspace{-2mm}
In this section, we discuss the theoretical and practical computational complexity of linear recurrence (with decay) when dealing with high-dimensional data, and then analyze the types of linear recurrence. In subsequent discussions, we assume $n$ is the sequence length, $d$ is the embedding dimension, and $\mathbf x_t\in \mathbb R^{d}$ is the transpose of the $t$-th row of matrix $\mathbf X\in \mathbb R^{n\times d}$.
\begin{figure}[t]
  \centering
  \setlength{\abovecaptionskip}{0.cm}
    \includegraphics[width=0.49\textwidth]{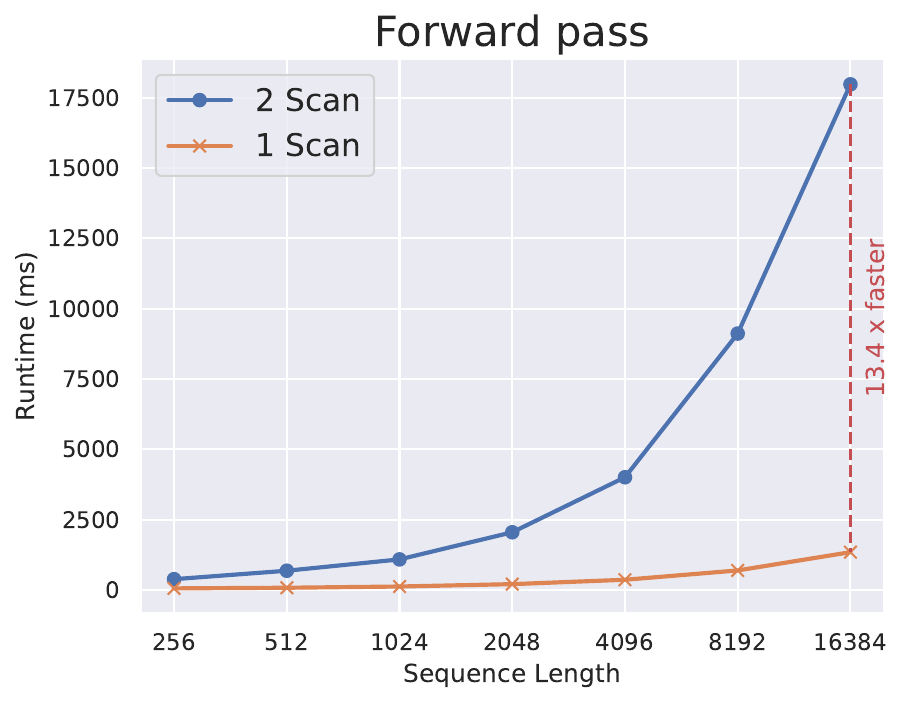}
    \includegraphics[width=0.49\textwidth]{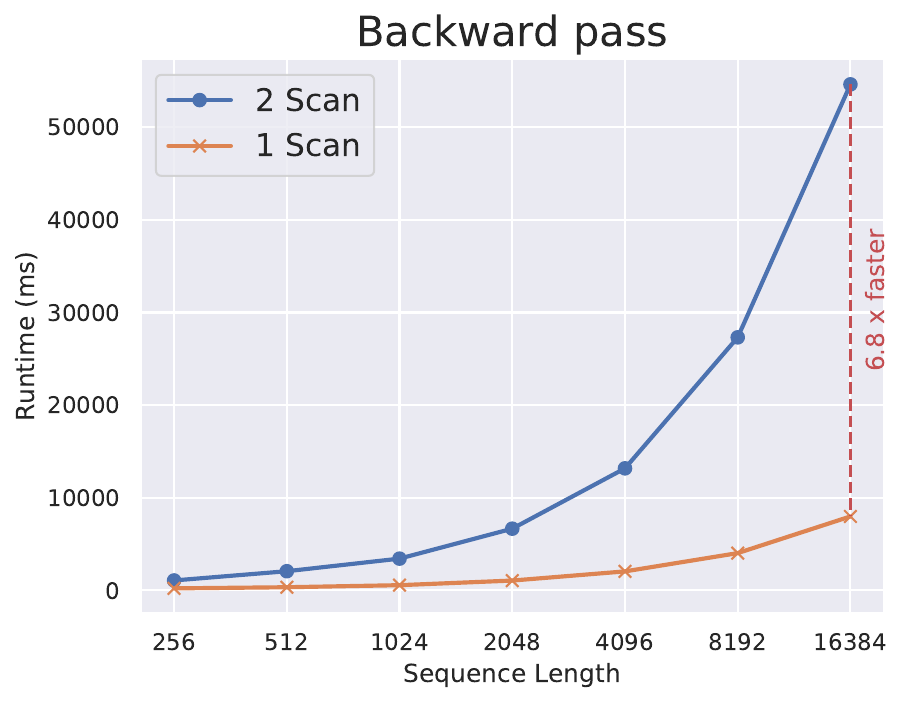}
    \caption{
     \textbf{Processing time of 1 Scan and 2 Scan in relation to sequence length.} 1 Scan is significantly faster than 2 Scan in both forward and backward passes. As the sequence length increases, the advantage of 1 Scan becomes more substantial.}
    \label{1vs2 scan}
    \vspace{-4mm}
\end{figure}

\subsection{Computational Complexity of Linear Recurrence}
It has been proved that all linear attention, state space model, and linear RNN can be expressed using a linear recurrence formula~\citep{qin2024unified}. We use linear attention with decay~\citep{2307.14995,retnet} as an example. Below is the recursive form (\ie~a scan):
$$
\mathbf{k} \mathbf{v}_0=\mathbf{0}, \mathbf{k} \mathbf{v}_t=\lambda_t \mathbf{k} \mathbf{v}_{t-1}+\mathbf{k}_t \mathbf{v}_t^{\top}, \mathbf{o}_t^{\top}=\mathbf{q}_t^{\top} \mathbf{k} \mathbf{v}_{\mathbf{t}},t=1,\ldots, n.
$$
Here, $0< \lambda_t \le 1$ is the decay rate. Note that the above formula is for the causal scenario. When dealing with non-causal scenarios, a common practice in the literature is to perform causal computation twice~\citep{duan_visionRWKV_2024,zhu_visionmamba_2024}. We call this method \enquote{2 scan}:
\begin{equation*}
\begin{gathered}
\overset{\rightarrow}{\mathbf{kv}}_0=\mathbf{0}, \overset{\rightarrow}{\mathbf{kv}}_t=\lambda_t \overset{\rightarrow}{\mathbf{kv}}_{t-1}+\mathbf{k}_t \mathbf{v}_t^{\top},
\overset{\rightarrow}{\mathbf{o}}_t^{\top}=\mathbf{q}_t^{\top}  \overset{\rightarrow}{\mathbf{kv}}_t,\\
\overset{\leftarrow}{\mathbf{kv}}_{n+1}=\mathbf{0}, \overset{\leftarrow}{\mathbf{kv}}_t=\lambda_t \overset{\leftarrow}{\mathbf{kv}}_{t+1}+\mathbf{k}_t \mathbf{v}_t^{\top},
\overset{\leftarrow}{\mathbf{o}}_t^{\top}=\mathbf{q}_t^{\top}  \overset{\leftarrow}{\mathbf{kv}}_t,\\
\mathbf o_t = \overset{\rightarrow}{\mathbf{o}}_t + \overset{\leftarrow}{\mathbf{o}}_t.
\end{gathered}
\end{equation*}
When $\lambda_t=1$, \ie~there is no decay, the right product trick~\citep{katharopoulos_icml_transformerrnn_icml_2020} can be used in this case. We call this method \enquote{1 scan}.
$$
[\mathbf {KV}] = \mathbf K^\top \mathbf V, \mathbf O = \mathbf Q [\mathbf {KV}].
$$
Although both of the above formulas have a time complexity of $O(nd^2)$, the \enquote{2 scan} version is significantly slower than the \enquote{1 scan} version. This is because causal computation requires block-level recursion~\citep{2401.04658,yang_GLA_2023}, whereas the second formula can be fully parallelized due to matrix multiplication~\citep{katharopoulos_icml_transformerrnn_icml_2020}. We provide a speed comparison in Fig.~\ref{1vs2 scan}, where the \enquote{2 scan} is implemented with Lightning Attention~\cite{2401.04658}, the fastest linear attention implementation so far. It can be seen that the \enquote{2 scan} is several times slower than the \enquote{1 scan} in both forward and backward passes.

It is apparent that the need for multiple scans is mainly due to the presence of decay $\lambda_t$. 
However, directly removing $\lambda_t$ would lead to degraded performance~\cite{qin_transnormer_emnlp_2022}. 
A natural question arises: \textit{can we retain $\lambda_t$ while only performing a single scan?}
In the next section, we will discuss the types of linear recurrence and answer the question.

\subsection{Types of Linear Recurrence}

We first explore the representation range of linear recurrences by 1D linear recurrence:
\begin{equation}
y_t = a_t y_{t-1} + x_t,  y_0 =0.
\label{equ:linear_Recurrence}
\end{equation}

Unroll the recursion equation of E.q~\ref{equ:linear_Recurrence}, we obtain:
\begin{equation}
y_t = \sum_{s=1}^{t} \frac{ A_s}{ A_t} x_s
\triangleq  \sum_{s=1}^{t} c_{ts} x_s
,A_t =\left (\prod_{s=1}^t a_s \right)^{-1}.
\label{equ:unroll_lr}
\end{equation}

The detailed proof of the unrolling process can be found in Appendix~\ref{proof:linear}.
Note that $y_t$ is a linear combination of $x_1, \ldots, x_t$. 
A natural question arises: \textit{Can every linear combination $\sum_{s=1}^{t} c_{ts} x_s$ be represented as a linear recursion?} 
We now prove that a linear recursion representation is possible only when the coefficients $c_{ts}$ satisfy certain conditions.

\newtheorem{linear-recurrence}{Theorem}[section]
\begin{linear-recurrence}
\label{thm:linear-recurrence}
A linear recurrence $y_t=a_t y_{t-1}+x_t, y_0=0$ is equivalent to a linear combination $y_t=\sum_{s=1}^t c_{ts} x_s$, iff $c_{ts}=\frac{g(s)}{g(t)}$, where $g(\cdot)$ is a function.
\end{linear-recurrence}

\begin{proof}[Proof of Theorem~\ref{thm:linear-recurrence}]
\label{proof1}
$\Rightarrow$

Given a linear recurrence, we multiply it by $A_t=\left(\prod_{s=1}^t a_s \right)^{-1}$ and the following recurrence equation:
\begin{equation*}
\begin{aligned}
A_t {y_t}=A_t a_t {y_{t-1}}+A_t  {x_t}&=A_{t-1} {y_{t-1}}+A_t x_t.
\end{aligned}
\end{equation*}
Unroll it, we get:
\begin{equation}
\label{equ:unroll_linear}
{A_t}{y_t}-{A_{t-1}}{y_{t-1}}={A_t} { x_t}, 
\ldots, 
{A_2}{y_2}-{A_{1}}{y_{1}} ={A_2}{x_2}.
\end{equation}
To derive an expression for $y_t$, we sum the recursive equations and obtain:
\begin{equation}
{A_t}{y_t} - {A_{1}}{y_{1}} =\sum_{s=2}^{t} { A_s} c{x_s}, \\
{y_t}{A_t} =\sum_{s=1}^{t} { A_s}{x_s}, 
y_t = \sum_{s=1}^{t} \frac{ A_s}{ A_t} x_s  .
\end{equation}
By comparing the coefficients, we can obtain $c_{ts}={A_s} / {A_t}$.

$\Leftarrow$:

Given the linear combination $y_t=\sum_{s=1}^t c_{ts} x_s$ and $c_{ts}=\frac{g(s)}{g(t)}$, we define $a_t \triangleq \frac{g(t-1)}{g(t)}$. Then $y_t$ can be expressed as:
\begin{equation*}
\begin{aligned}
y_t & =\sum_{s=1}^t c_{ts} x_s =\sum_{s=1}^{t-1} c_{ts} x_s  + c_{tt} x_t = \sum_{s=1}^{t-1} \frac{g(s)}{g(t)} x_s  +\frac{g(t)}{g(t)} x_t \\
&=\frac{g(t-1)}{g(t)}\sum_{s=1}^{t-1}   \frac{g(s)}{g(t-1)} x_s  + x_t =a_t \sum_{s=1}^{t-1}   c_{t-1, s} x_s + x_t  = a_t y_{t-1} + x_t.\qedhere
\end{aligned}
\end{equation*}
\end{proof}

Based on the Theorem~\ref{thm:linear-recurrence}, for linear recurrence, we can directly discuss $g(t)$, as $a_t$ can be obtained through $\frac{g(t-1)}{g(t)}$. Intuitively, $g(t)$ can be interpreted as an importance score up to moment $t$, $c_{ts}=\frac{g(s)}{g(t)}$ can be interpreted as the ratio of the score at moment $s$ relative to moment $t$, and $a_t$ can be interpreted as the ratio of the previous moment's score to moment $t$'s score.

Typically, to prevent numerical overflow, we assume $0 \le a_t=\frac{g(t-1)}{g(t)} \le 1$. To meet this condition, we present the following two forms:
\newtheorem{lrtype}[linear-recurrence]{Proposition}
\label{pro:linear-recurrence}
\begin{lrtype} 
For Linear Recurrence with $0 \le a_t \le 1$, there are two forms:

1. Multiplicative: $g(t)=\frac{1}{\prod_{s=1}^t\rho(s)},a_t=\rho(t), 0\le  \rho(t)\le 1$;

2. Additive: $g(t)=\sum_{s=1}^t\delta(s),a_t=\frac{\sum_{s=1}^{t-1} \delta(s)}{\sum_{s=1}^{t} \delta(s)}, \delta(s)\ge 0$.
\end{lrtype}
\begin{proof}[Proof of Proposition~\ref{pro:linear-recurrence}]
\label{proof2}
The condition $0\le \frac{g(t-1)}{g(t)} \le 1$, is equivalent to $\rho(t)=\frac{g(t-1)}{g(t)} \le 1$ or $\delta(t)=g(t)-g(t-1) \ge 0$.
By expanding the first formula, we obtain the multiplicative type. By expanding the second formula, we get the additive type.
\end{proof}

It can be observed that the typical linear attention with decay corresponds to the Multiplicative form, where $\rho(t)$ is utilized as $\text{Sigmoid}(\cdot)$~\citep{yang_GLA_2023}, $\text{exp}(-\text{exp}(\cdot))$~\citep{gu2023mamba}, or a fixed value~\citep{2307.14995, retnet}. Multiplicative requires multiple scans when dealing with high dimensions because $a_t$ itself cannot provide global information (as $a_t$ is only related to the moment $t$). However, for the Additive method, since the computation form is $a_t=\frac{\sum_{s=1}^{t-1} \delta(s)}{\sum_{s=1}^{t} \delta(s)}$, by modifying the denominator to $\Delta=\sum_{s=1}^n \delta(s)$ ($n$ is the sequence length), global information can be obtained through $a_t=\frac{\sum_{s=1}^{t-1} \delta(s)}{\Delta}$.

\section{LightNet}

Building upon the preceding analysis, we introduce a novel Linear Transformer architecture termed LightNet, designed to handle multi-dimensional data efficiently in a single scan. 
An overview of its structure is depicted in Fig.~\ref{arch}.
LightNet comprises an Input Embedding, MD-TPE module, and several stacked LightNet Layers.

\begin{figure}[t]
  \centering
  \setlength{\abovecaptionskip}{0.cm}
  \vspace{-2mm}
    \includegraphics[width=0.85\textwidth]{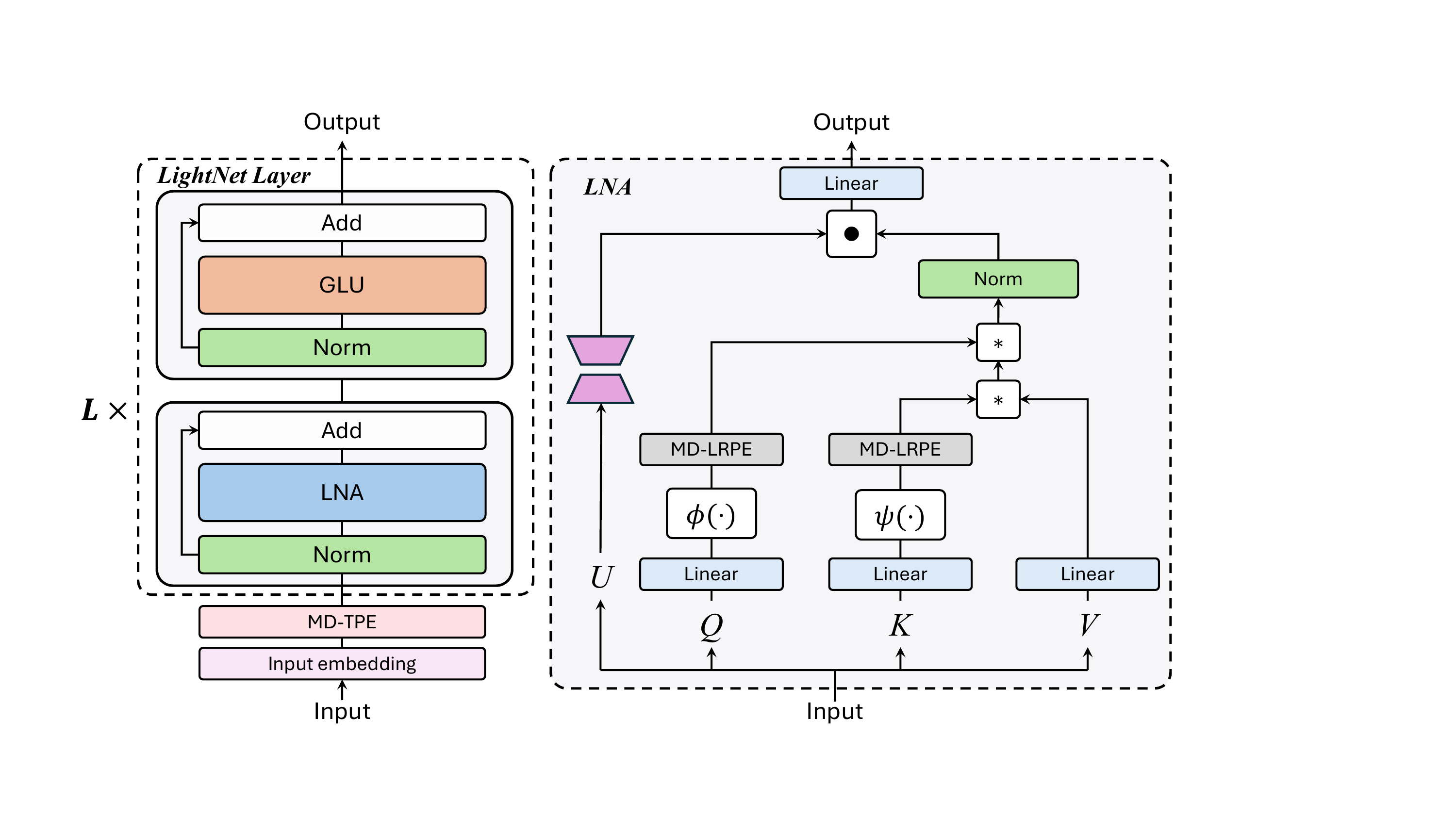}
    \caption{\textbf{The network structure of LightNet}: each LightNet model is comprised of an Input Embedding, MD-TPE, and a stack of multiple LightNet Layers. Each LightNet Layer consists of an LNA and a GLU, with the computation of LNA illustrated in the figure on the right.}
    \vspace{-4mm}
     \label{arch}
\end{figure}

\subsection{LightNet Layer}
The LightNet Layer is composed of a LightNet Attention (LNA) and a Gated Linear Unit (GLU)~\citep{2002.05202}. Within the LNA, an additive decay is employed, with the parameter $\delta$ implemented through the exponential function. Additionally, a parameter tiling strategy is utilized for both the key and decay, which has been empirically observed to enhance performance. This empirical evidence is detailed in Table~\ref{table:wiki103}. Furthermore, the integration of a low-rank output gate from TNL3~\citep{2307.14995} and a normalization after linear attention~\citep{qin_transnormer_emnlp_2022} has been incorporated.

In causal settings, the LightNet Layer can be represented as follows:
\begin{equation}
\begin{aligned}
\mathbf s_t = \mathbf s_{t-1} + \exp(\mathbf k_t ),
\mathbf {\bar k_t}&= \exp(\mathbf k_t ) / \mathbf s_t,  
\mathbf {kv}_t = \mathrm{diag}\{1-  \mathbf {\bar k_t}\} \mathbf {kv}_{t-1}  +  \mathbf {\bar k_t}\mathbf v_t^\top , \\
\mathbf o_t^\top &=  \mathrm{Norm}[\mathbf {kv}_t^\top \phi(\mathbf q_t)] \odot \psi(u_t) .
\end{aligned}
\end{equation}

In non-causal settings, the expression becomes: 
\begin{equation}
\label{eq:non-causal}
\begin{gathered}
\mathbf s=\sum_t \exp(\mathbf k_t ), 
\mathbf o_t =  \mathrm{Norm}\left[ \phi(q_t) \sum_t (\exp(k_t)/s)^\top \mathbf v_t \right] \odot \psi(u_t),  \\
\mathbf O = \mathrm{Norm} \left [ \phi(\mathbf Q)(\mathrm{Softmax}(\mathbf K)^\top \mathbf V) \right] \odot \psi(\mathbf U).
\end{gathered}
\end{equation}
where $\mathbf X$ is the input of LNA:
\begin{equation}
\mathbf{Q}=\mathbf{X} \mathbf{W}_q, \mathbf{K}=\mathbf{X} \mathbf{W}_k, \mathbf{V}=\mathbf{X} \mathbf{W}_v, \mathbf{U}=\mathbf{X} \mathbf{W}_{u1} \mathbf{W}_{u2},
\phi =\mathrm{Swish}, \psi=\mathrm{Sigmoid}.
\end{equation}

\subsection{Multi Dimension Position Encoding}
\label{sec:mdpe}
It is noted that additive recurrence does not have a locality prior like multiplicative recurrence and is permutation invariant in non-causal scenarios, as shown in E.q~\ref{eq:non-causal}. Therefore, it is necessary to introduce new positional encoding.
To tackle this challenge, we introduce two novel relative positional encoding methods, MD-TPE (Multi-Dimensional Toeplitz Positional Encoding) and expand the existing LRPE~\citep{qin2023linearized} to the high-dimensional context as MD-LRPE (Multi-Dimensional Linearized Relative Positional Encoding). This enhancement enables the management of relative positional relationships in any dimension.

\paragraph{MD-TPE.}
Given multi-dimension input $\mathbf x_{n_1,\ldots, n_k}, 1\le n_s \le N_s, s=1,\ldots k$, we use the following equation to capture positional information:
\begin{equation}
\begin{aligned}
\mathbf y_{n_1,\ldots, n_k}
=\sum_{m_k\le n_k} \ldots \sum_{m_1\le n_1}  \mathbf y_{n_1-m_1,\ldots, n_k-m_k}
\mathbf x_{m_1, \ldots, m_k}.
\end{aligned}
\end{equation}
 However, the time complexity of implementing the aforementioned method is $O(N\log N)$,where $ N=\prod_{s=1}^k n_s$, making it inefficient.
To address this, we simplified the above formula by performing toeplitz matrix production for each dimension separately and using SSM for parameterization~\citep{qin2023accelerating}, we denotes $e$ as the hidden dimension of SSM below:
\begin{equation}
\begin{aligned}
\mathbf y_{n_1,\ldots, n_k}
=\sum_{s=1}^k \sum_{m_s=1}^{n_s} \mathbf y_{n_s-m_s}
\mathbf x_{n_1,\ldots,m_s, \ldots, n_k}
=\sum_{s=1}^k  \sum_{m_s=1}^{n_s}\sum_{t=1}^e\lambda_{t}^{n_s-m_s} 
\mathbf x_{n_1,\ldots,m_s, \ldots, n_k}.
\end{aligned}
\end{equation}
By using a scan approach, the above calculation becomes linear in complexity, $O(Ne)$.

\paragraph{MD-LRPE.}
Given $\mathbf x_t \in \mathbb R^{d}, \mathbf x\in \{\mathbf q, \mathbf k\}$, LRPE transforms it through the matrix $\mathbf W_t$ to $\mathbf W_t \mathbf x_t, \mathbf x\in \{\mathbf q, \mathbf k\}$, and it holds that:
\begin{equation}
\begin{aligned}
(\mathbf W_s \mathbf q_s)^{\mathrm{H}} (\mathbf W_t \mathbf k_t)
=\mathbf q_s^{\mathrm{H}}\mathbf W_s ^{\mathrm{H}}\mathbf W_t \mathbf k_t
=\mathbf q_s^{\mathrm{H}}\mathbf W_{t-s} \mathbf k_t.
\end{aligned}
\end{equation}
We choose the complex version of LRPE, where:
\begin{equation}
\begin{aligned}
\mathbf W_t =\mathrm{diag}\{\exp(it\theta_1),\ldots, \exp(it\theta_d) \}.
\end{aligned}
\end{equation}
To generalize to higher dimensions, \ie, given $\mathbf x_{n_1,\ldots, n_k} \in \mathbb R^{d}, \mathbf x\in \{\mathbf q, \mathbf k\}$, we divide the $d$ features into $k$ groups, each group has $d/k$ features, with the $s$-th group's features corresponding to dimension $n_s$, $s\in [1, k]$. Specifically, we define:
\begin{equation}
\begin{aligned}
\mathbf W_{n_1,\ldots,n_k} =\mathrm{diag}\{[\Theta_1,\ldots, \Theta_k] \},
\Theta_s = \exp(i n_k \theta_{j}), sd/k< j \le  (s+1)d/k, 
\theta_j = 10000^{-2j / d}.
\end{aligned}
\end{equation}

Thus:
\begin{equation}
\begin{aligned}
\mathbf W_{n_1,\ldots,n_k}^{\mathbf H}\mathbf W_{m_1,\ldots,m_k}
=\mathbf W_{m_1-n_1,\ldots,m_k-n_k}
\end{aligned}
\end{equation}

Then:
\begin{equation}
\begin{gathered}
(\mathbf W_{n_1,\ldots, n_k} \mathbf q_{n_1,\ldots, n_k} )^{\mathrm{H}} (\mathbf W_{m_1,\ldots, m_k}  \mathbf k_{m_1,\ldots, m_k})
=\mathbf q_{n_1,\ldots, n_k}^{\mathrm{H}}\mathbf W_s ^{\mathrm{H}}\mathbf W_t \mathbf k_{m_1,\ldots, m_k} \\
=\mathbf q_{n_1,\ldots, n_k}^{\mathrm{H}}\mathbf W_{m_1-n_1,\ldots,m_k-n_k} \mathbf k_{m_1,\ldots, m_k}.
\end{gathered}
\end{equation}

\section{Experiments}
We comprehensively evaluate the substitutability of our LightNet in performance, scalability, flexibility, and efficiency.
We validate the effectiveness of our model on various multi-dimensional sequential modeling tasks.
We also test the proposed ability of LightNet to serve as a language model.

\begin{table}[ht]
\centering
\small
\caption{\textbf{Performance comparison for image classification task on ImageNet1k.} 
\enquote{S.A.} represents Softmax Attention, \enquote{M.S.} denotes multiple scans, and \enquote{O.S.} signifies one scan.}
\label{table:imagenet1k}
\begin{threeparttable}
\setlength{\tabcolsep}{0.1cm}
\begin{tabular}{lccccccc}
\toprule
\multirow{2}{*}{Model} & \multirow{2}{*}{Category} & \multicolumn{2}{c}{\textbf{Tiny}} & \multicolumn{2}{c}{\textbf{Small}} & \multicolumn{2}{c}{\textbf{Base}} \\
\cmidrule{3-8}
               &     & Acc (\%) $\uparrow$ & Params (M) & Acc (\%) $\uparrow$ & Params (M) & Acc (\%) $\uparrow$ & Params (M) \\
\midrule
DeiT~\cite{touvron_deit_icml_2021}             & S.A.   & 72.20 & 5.7   & 79.90 & 22.00   & 81.80 & 86.00 \\
HGRN~\cite{qin_hgrn_nips_2024}            & M.S.   & 74.40 & 6.1   & 80.09 & 23.70   & -    & - \\
Vim~\cite{zhu_visionmamba_2024}              & M.S.   & 76.10 & 7.0   & 80.50 & 26.00   & -    & - \\
V-RWKV~\cite{duan_visionRWKV_2024}            & M.S.   & 75.10 & 6.2   & 80.10 & 23.80   & 82.00 & 93.70 \\
\midrule
LightNet         & O.S.   & 74.46 & 6.0   & 80.12 & 22.64  & 81.90 & 87.74 \\
LightNet w/o  TPE & O.S. & 73.97 & 6.0 & 79.65 & 22.54 & 81.45 & 87.54 \\
LightNet w/o LRPE & O.S. & 74.02 & 6.0& 79.54 & 22.63 & 81.72 & 87.69 \\
\bottomrule
\end{tabular}
\end{threeparttable}
\end{table}

\subsection{Setting}

\begin{wraptable}[16]{r}{0.5\textwidth}
    \vspace{-8pt}
    \small
    \setlength{\tabcolsep}{0.01cm}
    \caption{\textbf{Performance comparison for image generation task on ImageNet-1k.} 
    LightNet-XL/2 achieves state-of-the-art FID with or without classifier-free guidance (-G)~\cite{ho_cfg_2022}.}
    \begin{threeparttable}
    \begin{tabular}{lccccc}
    \toprule
    Model & FID$\downarrow$   & sFID$\downarrow$  & IS$\uparrow$     & Precision$\uparrow$ & Recall$\uparrow$ \\
    \midrule
    CDM~\cite{ho_CDM_JMLR_2022} & 4.88 & - & 158.71 & - & - \\
    \midrule
    LDM-8~\cite{rombach_ldm_cvpr_2022} & 15.51 & - & 79.03 & 0.65 & 0.63 \\
    LDM-8-G & 7.76 & - & 209.52 & 0.84 & 0.35 \\
    LDM-4 & 10.56 & - & 103.49 & 0.71 & 0.62 \\
    LDM-4-G & 3.60 & - & 247.67 & 0.87 & 0.48 \\
    \midrule
    DiT-XL/2~\cite{peebles_dit_iccv_2023} & 9.62 & 6.85 & 121.50 & 0.67 & {0.67} \\
    DiT-XL/2-G & 2.27 & 4.60 & {278.24} & 0.83 & 0.57 \\
    \midrule
    LightNet-XL/2               & 5.35 & 5.93 & 171.18 & 0.73 & 0.65 \\
    LightNet-XL/2-G             & 2.18 & 4.58 & 281.85 & 0.83 & 0.58 \\
    \bottomrule
    \end{tabular}
    \end{threeparttable}
    \vspace{-20pt}
    \label{table:image_generation}
\end{wraptable}
\vspace{-3mm}
\paragraph{Image Classification.}
We trained our LightNet model for image classification on the ImageNet-1K dataset~\cite{deng2009imagenet}. Our approach modifies the network architecture and training protocols of DeiT~\cite{touvron_deit_icml_2021}, substituting its Transformer Layers with our proprietary LightNet Layers.

\vspace{-3mm}
\paragraph{Image Generation.}
We build our model upon the latent diffusion model~\cite{rombach_ldm_cvpr_2022,peebles_dit_iccv_2023} and use our proposed LightNet as the denoising network.
We adjust the model size across various configurations (S, B, L, XL) and patch sizes (8, 4, 2), consistent with DiT~\cite{peebles_dit_iccv_2023}.
Experiments are conducted on the ImageNet dataset~\cite{deng2009imagenet} at a resolution of $256 \times 256$. 
Each model is trained over 0.4M steps with a batch size of 256 to assess scaling capabilities. 
For the largest model variant, training is extended to 0.8M steps with a batch size of 1024, as opposed to the 7M steps in DiT, to enhance generative performance.

\begin{wraptable}[24]{r}{0.5\textwidth}
    \vspace{-20pt}
    \centering
    \small
    \caption{\textbf{Performance comparison on Wikitext-103}. $\downarrow$ means \textit{lower is better}. We adopted the configuration of HGRN for Wikitext-103, and we can observe that LightNet significantly outperforms all other methods.}
    \begin{threeparttable}
    \setlength{\tabcolsep}{0.3cm}
    \begin{tabular}{llll}
    \toprule
        Model & \makecell[c]{PPL \\(val) $\downarrow$} & \makecell[c]{PPL \\(test) $\downarrow$} & \makecell[c]{Params\\(M)} \\ 
    \midrule
        \textit{Attn-based}  \\ 
        Transformer & 24.40 & 24.78 & 44.65 \\ 
        FLASH & 25.92 & 26.70 & 42.17 \\ 
        1+elu & 27.44 & 28.05 & 44.65 \\ 
        Performer & 62.50 & 63.16 & 44.65 \\
        cosFormer & 26.53 & 27.06 & 44.65 \\ 
    \midrule
        \textit{RNN-based}  \\ 
        S4 & 38.34 & 39.66 & 45.69 \\ 
        DSS & 39.39 & 41.07 & 45.73 \\ 
        GSS & 29.61 & 30.74 & 43.84 \\ 
        RWKV-4 & 24.31 & 25.07 & 46.23\\
        LRU & 29.86 & 31.12 & 46.24\\
        {HGRN} & {24.14} & 24.82 &46.25 \\ 
    \midrule
        \textit{FFT-based}  \\ 
        TNN & {23.98} & {24.67} & 48.68 \\  \midrule
        LightNet & {23.09} & 23.75 &45.07 \\
    \bottomrule
    \end{tabular}
    \end{threeparttable}
    \label{table:wiki103}
    \vspace{-40pt}
\end{wraptable}

\vspace{-3mm}
\paragraph{Bidirectional Language Modeling.}
We utilize Cramming-BERT~\citep{2212.14034} as our pipeline, employing a 24-hour training regime to pre-train on the Pile dataset, subsequently finetuning on the GLUE benchmark~\citep{1804.07461}. 
During pre-training, we follow established guidelines by setting a learning rate of 1e-3, a sequence length of 128, and a batch size of 8192. 
In the finetuning phase, we experiment with learning rates from the set \{5e-5, 4e-5, 3e-5, 2e-5\} and determine the optimal outcome by finetuning over 5 epochs.
\vspace{-3mm}
\paragraph{Autoregressive Language Modeling.}
We evaluate two capabilities: perplexity (PPL) and zero-shot reasoning ability. 
The perplexity of the 44M model is assessed on the Wikitext-103 dataset \cite{merity_Wikitext103_iclr_2016}, and the 380M model's perplexity is tested on the Pile dataset, consuming 10 billion tokens . 
For large language model experiments, we train LightNet models at scales of 150M, 350M, 1B, and 3B using 100 billion tokens sampled from subsets of the Pile~\citep{pile}. 
These models are then evaluated on commonsense reasoning tasks using the lm-eval-harness~\citep{eval-harness}. Detailed training hyperparameters are listed in Table~\ref{table:config}.

\vspace{-2mm}
\subsection{Results}
\vspace{-2mm}

\paragraph{Image Classification.}
As shown in Table~\ref{table:imagenet1k}, the proposed LightNet shows competitive performance on the ImageNet-1k dataset. 
It can be observed that using only a single sequential scan, LightNet can achieve comparable performance to models with naive attention and multiple sequential scans.
At the same time, the speed advantage of using a single sequential scan is shown in Fig.~\ref{1vs2 scan}.

\vspace{-4mm}
\paragraph{Image Generation.}
The image generation results are presented in Table~\ref{table:image_generation}. 
Our proposed LightNet demonstrates superior performance, achieving a lower Fréchet Inception Distance (FID) and a higher Inception Score (IS) than DiT~\cite{peebles_dit_iccv_2023} with fewer training steps (0.8M steps vs 7M steps). 
Additionally, LightNet exhibits commendable scaling capabilities, as illustrated in Fig.~\ref{diffusion_scaling}.

\vspace{-4mm}
\begin{wraptable}[14]{r}{0.45\textwidth}
\vspace{-18pt}
\label{table:lm_large}
    \centering
    \small
    \caption{\textbf{Performance comparison Pile for large-scale language modeling.}. We trained under the 10 billion token subset of Pile, and it can be seen that LightNet's PPL is better than LLaMA's.}
    \begin{threeparttable}
    \setlength{\tabcolsep}{0.35cm}
    \begin{tabular}{lll}
    \toprule
        Model & PPL $\downarrow$ & Params \\
    \midrule
        LLaMA & 4.62 & 385M \\
          TNL & 4.62 & 379M \\
          Mamba & 4.59 & 385M \\
        LightNet & 4.59 & 379M \\
    \midrule
        LightNet w/o TPE & 4.69 & 379M \\
        LightNet w/o LRPE & 4.69 & 379M \\
        LightNet no share & 4.76 & 385M \\
    \bottomrule
    \end{tabular}
    \end{threeparttable}
    \label{tab:exp_llm}
\vspace{-20pt}
\end{wraptable}
\paragraph{Bidirectional Language Modeling.}
As shown in Table~\ref{table:glue}, LightNet outperforms Crammed Bert~\citep{2212.14034} on the GLUE dataset, demonstrating its superior capability in handling natural language understanding tasks. Despite BERT-Base~\citep{devlin-etal-2019-bert} achieving comparable performance, it is noteworthy that LightNet does so with a significantly lower computational cost, having been trained on a single A100 for 24 hours.

\vspace{-4mm}
\paragraph{Autoregressive Language Modeling.}
In the Wikitext-103 dataset, as depicted in Table~\ref{table:wiki103}, LightNet surpasses all competitors on both the validation and test datasets. Regarding large-scale datasets, as illustrated in Table~\ref{tab:exp_llm}, LightNet exhibits superior perplexity (PPL) compared to LLaMA~\citep{llama} and TNL~\citep{2307.14995}, and matches the performance of Mamba~\citep{gu2023mamba}. The ability of LightNet to achieve high performance with reduced parameter complexity underscores its potential for scalability and broader application across various large-scale data scenarios.

\begin{figure}[t]
  \centering
  \setlength{\abovecaptionskip}{0.cm}
  \vspace{-3mm}
    \includegraphics[width=0.45\textwidth]{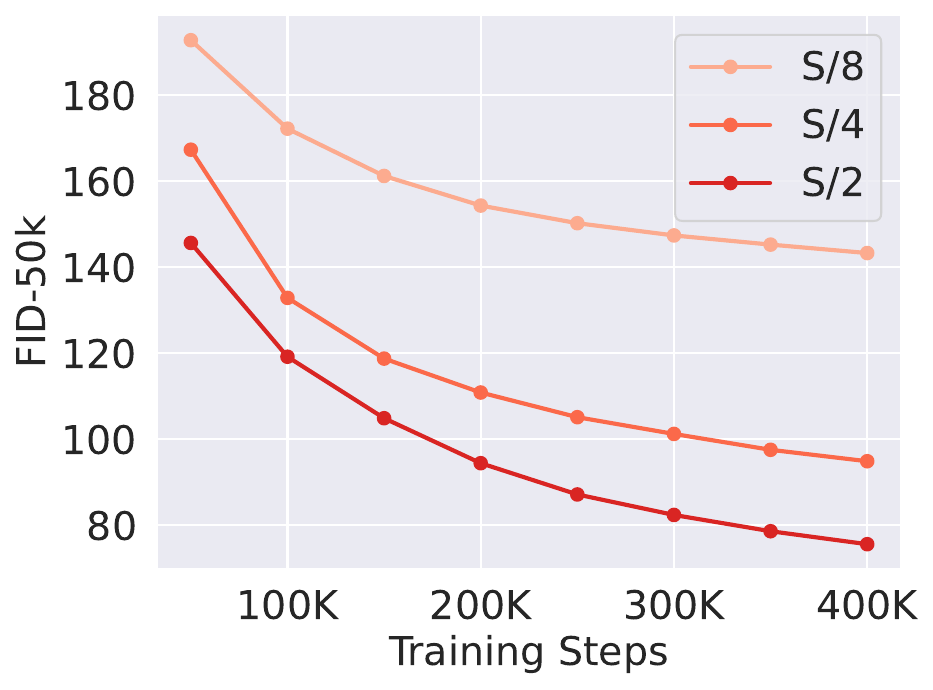}
    \includegraphics[width=0.45\textwidth]{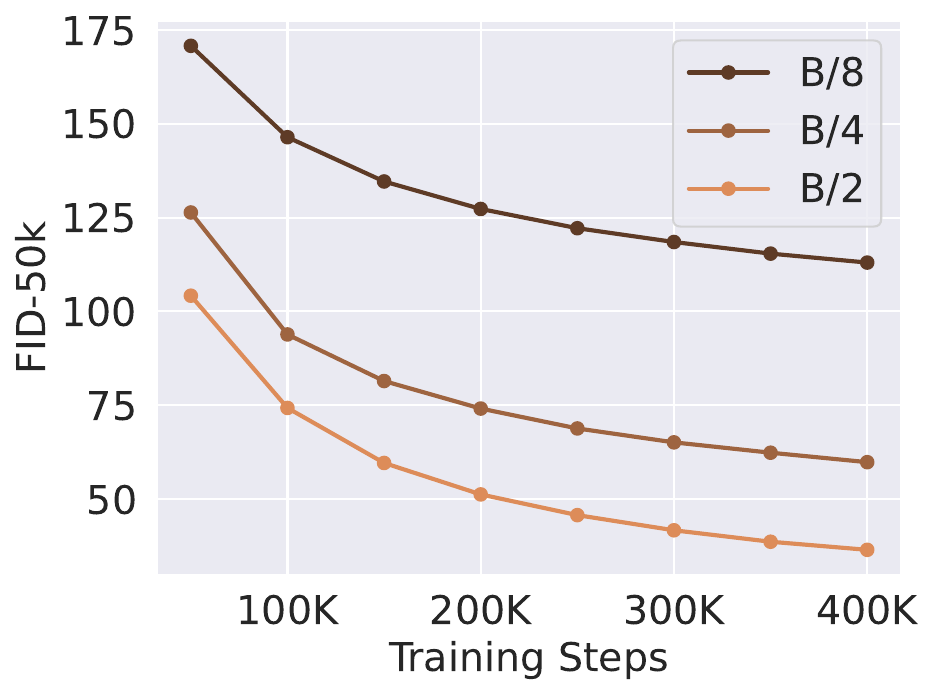}
    \includegraphics[width=0.45\textwidth]{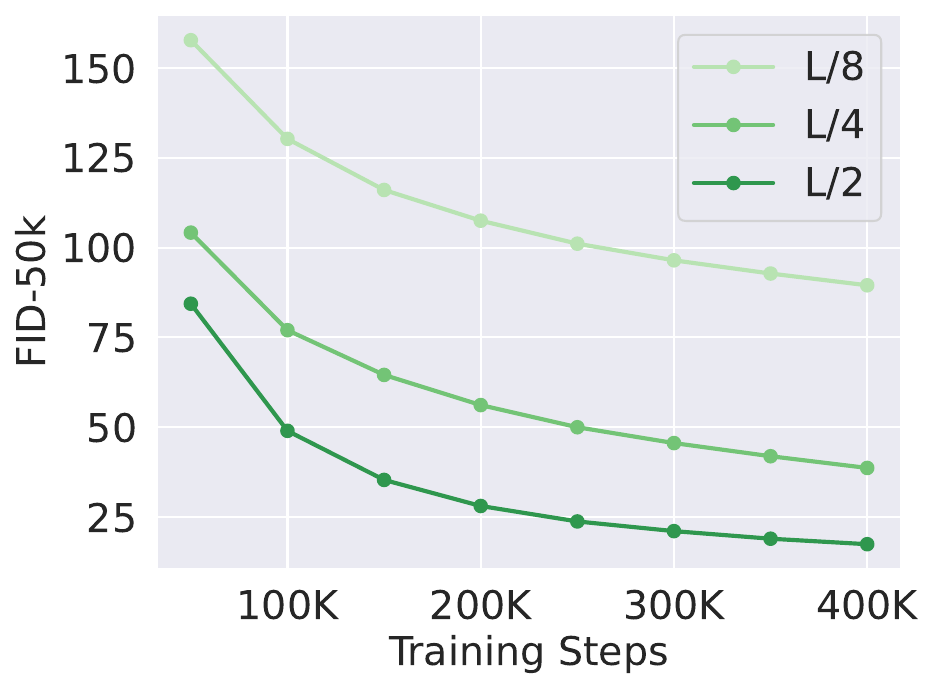}
    \includegraphics[width=0.45\textwidth]{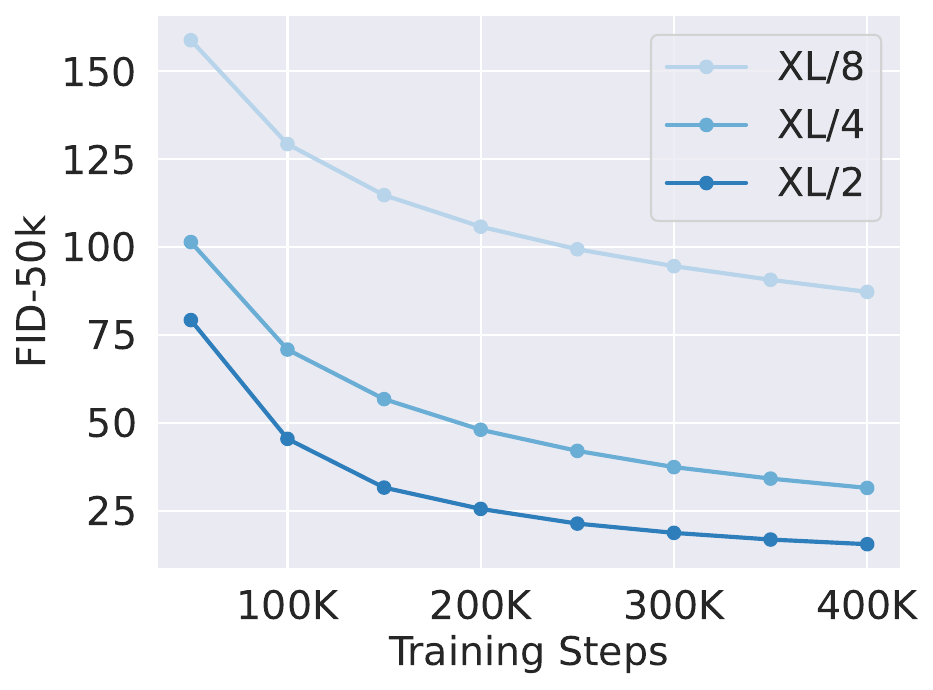}
    \caption{\textbf{Scaling up the LightNet enhances the FID during every stages of training.} We present the FID-50K across training iterations for twelve LightNet models. Enhancing the LightNet backbone results in improved generative models for all sizes of models and patches.}
    \vspace{-4mm}
    \label{diffusion_scaling}
\end{figure}

\begin{table}[t]
\centering
\caption{\textbf{Performance Scores on GLUE Benchmark.} We utilize the Cramming-BERT 24-hour training configuration and observe that LightNet outperforms Crammed BERT and achieves comparable results to BERT-Base, which is trained with more GPU hours.}
\label{table:glue}
\begin{threeparttable}
\setlength{\tabcolsep}{0.14cm} 
\small 
\begin{tabular}{lccccccccc}
\toprule
\textbf{Model} & \textbf{MNLI} & \textbf{SST-2} & \textbf{STSB} & \textbf{RTE} & \textbf{QNLI} & \textbf{QQP} & \textbf{MRPC} & \textbf{CoLA} & \textbf{GLUE} \\
\midrule
BERT-Base (Fully trained) & 83.2 / 83.4 & 91.9 & 86.7 & 59.2 & 90.6 & 87.7 & 89.3 & 56.5 & 80.9 \\
BERT-Base (No Pretrain)   & 34.1 / 34.1 & 79.9 & 17.8 & 47.3 & 50.0 & 68.6 & 77.9 & - & 45.5 \\
\midrule
Crammed BERT      & 83.9 / 84.1 & 92.2 & 84.6 & 53.8 & 89.5 & 87.3 & 87.5 & 44.5 & 78.6 \\
\midrule
LightNet                  & 83.3 / 83.5 & 92.9 & 86.3 & 55.6 & 89.1 & 87.7 & 88.5 & 52.6
& 79.9\\
LightNet w/o TPE   & 82.1 / 82.9 & 	92.4&  79.4&  57.8&  89.2& 87.7& 
 83.8&  44.1&  77.7 \\
\bottomrule
\end{tabular}
\end{threeparttable}
\end{table}

\subsection{Ablation Studies}

\begin{table}[t]
    \centering
    \small
    \begin{threeparttable}
    \setlength{\tabcolsep}{0.28cm}
    \caption{\textbf{Ablation studies on image generation} for LightNet-B/2 Configurations. 
    We compare the performance of FID under different training steps.}
    \label{tab:lightnet-b2-performance}
    \begin{tabular}{lcccccccc}
    \toprule
        Model & 50K & 100K & 150K & 200K & 250K & 300K & 350K & 400K \\
    \midrule
        LightNet-B/2 & 104.19 & 74.27 & 59.60 & 51.22 & 45.70 & 41.65 & 38.60 & 36.45 \\
        LightNet-B/2 w/o TPE & 105.86 & 77.64 & 64.81 & 57.24 & 51.98 & 48.12 & 44.90 & 42.74 \\
       LightNet-B/2 w/o LRPE & 132.17 & 82.99 & 67.79 & 59.02 & 52.88 & 48.41 & 44.94 & 42.37 \\
    \bottomrule
    \end{tabular}
    \end{threeparttable}
    \vspace{-4mm}
\end{table}

\noindent
\textbf{Effectiveness of Parameters Sharing.}
As discussed in Sec.~\ref{sec:mdpe}, we employ a parameter sharing strategy between decay and key, and the performance comparison is presented in Table \ref{tab:exp_llm}. 
The results demonstrate that employing independent parameters for decay and key leads to performance deterioration, highlighting the significance of parameter sharing.

\textbf{Effectiveness of MD-TPE.}
The proposed MD-TPE provides relative positional information under linear complexity.
We thus explore the effectiveness of the MD-TPE across all tasks, shown in Table~\ref{table:imagenet1k},\ref{tab:exp_llm},\ref{table:glue},\ref{tab:lightnet-b2-performance}.
We can observe that removing MD-TPE results in significant performance degradation, particularly for image generation, which highly depends on the relative position of the image content. 
Similarly, performance comparison in language modeling tasks also confirms the effectiveness of MD-TPE when reduced to a single dimension.

\textbf{Effectiveness of MD-LRPE.}
LRPE has already proven its effectiveness in the field of language modeling. Therefore, when faced with higher-dimensional inputs, the contributions of its extension, MD-LRPE should be systematically validated. To this end, we conduct numerous ablation experiments, and the results, as shown in Table~\ref{table:imagenet1k},\ref{tab:exp_llm},\ref{table:glue},\ref{tab:lightnet-b2-performance}, demonstrate the effectiveness of extending LRPE into a multi-dimensional space through MD-LRPE operation.

\textbf{Speed Test.} 
The current linear complexity models employ multiplicative linear recurrence in sequence modeling and necessitate at least two scans for multi-dimensional data, resulting in processing time denoted by the \enquote{2 Scan} in Fig.~\ref{1vs2 scan}. In contrast, our LightNet requires only a single scan, leading to a processing time denoted by the \enquote{1 Scan}. As evident from the figure, the advantage of the \enquote{1 Scan} becomes increasingly pronounced with the growth of sequence length.

\section{Conclusion}
\vspace{-2mm}

In this paper, we have addressed the inefficiency of "multiplicative" linear recurrence in multi-dimensional sequence modeling by introducing a novel "additive" linear recurrence that handles multi-dimensional data within a single scan. We developed LightNet, a new multi-dimensional linear attention model enhanced by two new multi-dimensional linear relative positional encoding methods, MD-TPE and MD-LRPE. Empirical evaluations across tasks like image classification, image generation, bidirectional language modeling, and autoregressive language modeling demonstrate LightNet's superior performance and versatility. LightNet offers a significant advancement in efficiency and scalability, providing a promising pathway for future research and applications in multi-dimensional sequence modeling.

\textbf{Limitations.}
Our empirical evaluation of LightNet is conducted on a smaller scale compared to other large-scale models. Potential negative social consequences include the misuse of brain models for inappropriate purposes or applications, necessitating prohibition through appropriate regulations.

\section{Acknowledgments}
This research was supported in part by the National Natural Science Foundation of China (62271410), the National Key R\&D Program of China (NO.2022ZD0160100), and the Fundamental Research Funds for the Central Universities. Yuxin Mao is sponsored by the Innovation Foundation for Doctor Dissertation of Northwestern Polytechnical University (CX2024014).

{\small
\bibliographystyle{unsrt}
\bibliography{lightnet}
}

\appendix
\newpage
\section{Appendix}
\subsection{Proof of Eq~\ref{equ:unroll_lr}}
\label{proof:linear}
Note that
\begin{equation*}
\begin{gathered}
A_t {y_t}=A_t a_t {y_{t-1}}+A_t  {x_t}=A_{t-1} {y_{t-1}}+A_t x_t,  \\
{A_t}{y_t}-{A_{t-1}}{y_{t-1}}={A_t} { x_t}, \\
\ldots, \\
{A_2}{y_2}-{A_{1}}{y_{1}} ={A_2}{x_2} .
\end{gathered}
\end{equation*}
By summing up, we can obtain:
\begin{equation*}
{A_t}{y_t} - {A_{1}}{y_{1}} =\sum_{s=2}^{t} { A_s} c{x_s}, \\
{y_t}{A_t} =\sum_{s=1}^{t} { A_s}{x_s},\\
y_t = \sum_{s=1}^{t} \frac{ A_s}{ A_t} x_s  .
\end{equation*}

\subsection{More experiments}
In this section, we provide additional experimental results. In Table~\ref{exp:llm}, we show the performance of LightNet under the Commonsense Reasoning Tasks. In Table~\ref{tab:lightnet-scores}, we present the effects of LightNet on image generation tasks across various sizes.
\begin{table}[htb!]
\small
\caption{\textbf{Performance Comparison on Commonsense Reasoning Tasks.} PS, T, HS, WG stand for parameter size (billion), tokens (billion), HellaSwag, and WinoGrande, respectively. }
    \centering
    \label{exp:llm}
    \begin{tabular}{l|c|c|l|l|l|l|l|l|l}
    \toprule
        Model & P & T & PIQA & HS & WG & ARC-e & ARC-c & OBQA & AVG \\ \midrule
        GPT-Neo & 0.13 & 300 & 63.06  & 30.40  & 50.43  & 43.73  & 23.12  & 26.20  & 39.49  \\ 
        OPT & 0.16 & 300 & 62.95  & 31.35  & 50.43  & 43.52  & 22.70  & 28.00  & 39.83  \\ 
        Pythia & 0.16 & 300 & 61.32  & 30.16  & 51.93  & 43.18  & 23.12  & 26.80  & 39.42  \\ 
        RWKV-4 & 0.17 & - & 65.07  & 32.26  & 50.83  & 47.47  & 24.15  & 29.60  & 41.56  \\ 
        HGRN & 0.15 & 100 & 65.02  & 33.33  & 50.20  & 46.68  & 23.81  & 28.60  & 41.27  \\ 
        
        LightNet & 0.15 & 100 & 63.49& 30.55& 50.83	& 46.04&  24.40& 27.40&  40.45 \\
        \midrule
        OPT & 0.35 & 300 & 64.58  & 36.69  & 52.49  & 44.02  & 23.89  & 28.20  & 41.65  \\ 
        Pythia & 0.40 & 300 & 67.08  & 40.52  & 53.59  & 51.81  & 24.15  & 29.40  & 44.43  \\ 
        BLOOM & 0.56 & 350 & 64.09  & 36.97  & 52.80  & 47.35  & 23.98  & 28.20  & 42.23  \\ 
        RWKV-4 & 0.43 & - & 67.52  & 40.90  & 51.14  & 52.86  & 25.17  & 32.40  & 45.00  \\ 
        HGRN & 0.35 & 100 & 66.70  & 38.12  & 51.70  & 49.20  & 25.26  & 30.60  & 43.60  \\ 
        LightNet & 0.39 & 100 & 66.87&  38.82&  50.75& 51.39& 25.17	& 28.20&  43.53 \\
        \midrule
        GPT-Neo & 1.3 & 300 & 71.11  & 48.93  & 54.93  & 56.19  & 25.85  & 33.60  & 48.44  \\ 
        OPT & 1.3 & 300 & 71.71  & 53.70  & 59.35  & 57.24  & 29.69  & 33.20  & 50.82  \\ 
        Pythia & 1.4 & 300 & 70.67  & 47.18  & 53.51  & 56.99  & 26.88  & 31.40  & 47.77  \\ 
        BLOOM & 1.1 & 350 & 67.14  & 42.98  & 54.93  & 51.47  & 25.68  & 29.40  & 45.27  \\
        RWKV-4 & 1.5 & - & 72.36  & 52.48  & 54.62  & 60.48  & 29.44  & 34.00  & 50.56  \\ 
        HGRN & 1.0 & 100 & 70.89  & 48.02  & 51.62  & 55.64  & 27.90  & 31.60  & 47.61  \\ 
        LightNet & 1.0 & 100 & 69.75& 44.68& 51.85& 55.35&  27.65& 32.80& 47.01  \\
        
        \midrule
        OPT & 2.7 & 300 & 73.83 & 60.60 & 61.01 & 60.77 & 31.31 & 35.2.0 & 53.79  \\ 
       Pythia & 2.8 & 300 & 74.10 & 59.31 & 59.91 & 64.14 & 33.02 & 35.60 & 54.35  \\ 
        BLOOM & 3.0 &  350& 70.57 & 54.53 & 58.48 & 59.43 & 30.38 & 32.20 & 50.93  \\ 
       RWKV-4 & 3.0 & - & 72.42 & 58.75 & 57.30 & 62.92 & 35.15 & 36.20 & 53.79  \\
         LightNet  & 2.9 & 100 & 73.56& 55.47& 57.77& 61.57& 32.94& 33.60& 52.49 \\
    \bottomrule
    \end{tabular}
    \label{table:lm-eval}
\end{table}

\begin{table}[!ht]
    \centering
    \small
    \begin{threeparttable}
    \caption{Performance Metrics Across Different LightNet Configurations}
    \label{tab:lightnet-scores}
    \setlength{\tabcolsep}{0.1cm} 
    \begin{tabular}{lccccccccc}
    \toprule
        Model & 50K & 100K & 150K & 200K & 250K & 300K & 350K & 400K \\
    \midrule
        LightNet-S/8 & 192.79 & 172.23 & 161.23 & 154.34 & 150.25 & 147.40 & 145.27 & 143.31 \\
        LightNet-S/4 & 167.33 & 132.89 & 118.77 & 110.88 & 105.15 & 101.25 & 97.56 & 94.90 \\
        LightNet-S/2 & 145.66 & 119.20 & 104.90 & 94.45 & 87.18 & 82.41 & 78.63 & 75.61 \\
        DiT-S/2 & - & - & - & - & - & - & - & 67.16 \\
    \midrule
        LightNet-B/8 & 170.79 & 146.43 & 134.63 & 127.31 & 122.18 & 118.50 & 115.40 & 113.02 \\
        LightNet-B/4 & 126.37 & 93.86 & 81.44 & 74.11 & 68.80 & 65.09 & 62.34 & 59.81 \\
        LightNet-B/2 & 104.19 & 74.27 & 59.60 & 51.22 & 45.70 & 41.65 & 38.60 & 36.45 \\
        DiT-B/2 & - & - & - & - & - & - & - & 42.76 \\
    \midrule
        LightNet-L/8 & 157.76 & 130.29 & 116.06 & 107.50 & 101.10 & 96.47 & 92.79 & 89.51 \\
        LightNet-L/4 & 104.18 & 77.02 & 64.55 & 56.16 & 49.99 & 45.58 & 41.91 & 37.54 \\
        LightNet-L/2 & 84.38 & 48.98 & 35.32 & 28.05 & 23.75 & 21.06 & 18.94 & 17.42 \\
        DiT-L/2 & - & - & - & - & - & - & - & 24.37 \\
    \midrule
        LightNet-XL/8 & 158.75 & 129.23 & 114.72 & 105.75 & 99.35 & 94.53 & 90.66 & 87.22 \\
        LightNet-XL/4 & 101.39 & 70.84 & 56.75 & 48.04 & 42.04 & 37.43 & 34.16 & 31.51 \\
        LightNet-XL/2 & 79.22 & 45.46 & 31.61 & 25.55 & 21.37 & 18.74 & 16.84 & 15.52 \\
        DiT-XL/2 & - & - & - & - & - & - & - & 19.20 \\
    \bottomrule
    \end{tabular}
    \end{threeparttable}
\end{table}

\subsection{Configurations}
In this section, we provide training configurations for all experiments. The configuration for Bidirectional Language Modeling is the same as~\citep{2212.14034}, while the configurations for the other experiments are as shown in Table~\ref{table:config},~\ref{config:llm},~\ref{config:ig},~\ref{config:im}. 
We use Pytorch~\citep{1912.01703} and A100 for training.

\begin{table*}[!ht]
\small
\center
\setlength{\tabcolsep}{0.5cm}
{
\caption{\textbf{Comprehensive Configurations of the Model and Training Procedures for LightNet Experiments} ``Total batch size'' means $\mathrm{batch\_per\_gpu} \times \mathrm{update\_freq} \times \mathrm{num\_gpus}$; ``ALM'' stands for Autoregressive Language Model; ``IM'' stands for Image Modeling, ``IG`` stands for image generation.}
\label{table:config}
\begin{tabular}{l|l|l|l}
\toprule
 & ALM  & IM & IG   \\
\midrule
Dataset                                              & WikiText-103    & ImageNet-1k   & ImageNet-1k      \\
Tokenizer method & BPE    & -  & -  \\
Src Vocab size & 50265   & - & -\\
Sequence length    & 512   & -  & -       \\
Total batch size & 128   &
2048  &
256  \\
Number of updates/epochs                            & 50k updates   & 300 epochs & 80 epochs   \\
Warmup steps/epochs    & 4k steps    & 20 epochs   & -   \\
Peak learning rate   & 5e-4     &5e-4  &1e-4                          \\
Learning rate scheduler  & Inverse sqrt    &Cosine  &-                   \\
Optimizer   & Adam  &Adamw    &Adamw                                    \\
Adam $\epsilon$   & 1e-8   & 1e-8  & 1e-8                           \\
Adam $(\beta_1,\beta_2)$                          & (0.9, 0.999)   & (0.9, 0.98)   & (0.9, 0.98)           \\
Weight decay       & 0.1 & 0.1 for Base, else 0.05 & 0\\                                       
Gradient clipping                                            &  -  & 5.0        & -                                       \\
GPUS & 4 & 8 & 8 \\
\bottomrule
\end{tabular}}

\end{table*}

\begin{table}[!ht]

\caption{
\textbf{Configurations for LLM}}
\label{config:llm}
\small
    \centering
    \begin{tabular}{cccccccc}
    \toprule
        Params(B) & Layers & Hidden Dim  & L.R. & Batch Size Per GPU& SeqLen & GPUs \\ \midrule
        0.15 & 15 & 768  & 3.00E-04 & 26 & 2048 & 8 \\
        0.385 & 26 & 1024& 3.00E-04 & 15 & 2048 & 8 \\ 
        1.0 & 18 & 2048  & 3.00E-04 & 10 & 2048 & 16 \\ 
        2.9 & 36 & 2560 & 3.00E-04 & 36 & 2048 & 48 \\ 
    \bottomrule
    \end{tabular}

\end{table}

\begin{table}[!ht]
\caption{
\textbf{Model Configurations for Image Generation task.}}
\label{config:ig}
    \centering
    \begin{tabular}{ccccc}
    \hline
        Model & Layers & Hidden Dim & Heads & Params \\ \hline
        LightNet-S & 18 & 384 & 6 & 33M \\ 
        LightNet-B & 18 & 768 & 6 & 131M \\ 
        LightNet-L & 36 & 1024 & 16 & 470M \\ 
        LightNet-XL & 42 & 1152 & 16 & 680M \\ \hline
    \end{tabular}
\end{table}

\begin{table}[!ht]
\caption{
\textbf{Model Configurations for Image Classification task.}}
\label{config:im}
    \centering
    \begin{tabular}{ccccc}
    \hline
        Model & Layers & Hidden size & Heads & Params \\ \hline
        LightNet-T & 12 & 192 & 6 & 6.0M \\ 
        LightNet-S & 12 & 384 & 16 & 22.6M \\ 
        LightNet-B & 12 & 768 & 16 & 87.7M \\ \hline
    \end{tabular}
\end{table}

\end{document}